\documentclass[conference]{IEEEtran}
\IEEEoverridecommandlockouts
\usepackage{cite}
\usepackage{amsmath,amssymb,amsfonts,amsthm}
\usepackage{algorithmic}
\usepackage{graphicx}
\usepackage{textcomp}
\usepackage{xcolor}
\def\BibTeX{{\rm B\kern-.05em{\sc i\kern-.025em b}\kern-.08em
    T\kern-.1667em\lower.7ex\hbox{E}\kern-.125emX}}

\usepackage{thmtools}
\usepackage{thm-restate}

\newtheorem{remark}{Remark}%

\usepackage{algorithmic}
\usepackage[ruled,norelsize]{algorithm2e}

\usepackage[caption=false,font=normalsize,labelfont=sf,textfont=sf]{subfig}

\begin{document}

\title{Trajectory-Oriented Policy Optimization with Sparse Rewards\\
}

\author{\IEEEauthorblockN{1\textsuperscript{st} Guojian Wang}
\IEEEauthorblockA{\textit{School of Mathematical Sciences} \\
\textit{Beihang University}\\
Beijing, China \\
wgj@buaa.edu.cn}
~\\
\and
\IEEEauthorblockN{2\textsuperscript{nd} Faguo Wu*}
\IEEEauthorblockA{\textit{Institute of Artificial Intelligence} \\
\textit{Beihang University}\\
Beijing, China \\
faguo@buaa.edu.cn}
*Corresponding author
~\\
\and
\IEEEauthorblockN{3\textsuperscript{rd} Xiao Zhang}
\IEEEauthorblockA{\textit{School of Mathematical Sciences} \\
\textit{Beihang University}\\
Beijing, China \\
xiao.zh@buaa.edu.cn}
*Corresponding author
~\\
}


\maketitle

\begin{abstract}
The mastery of deep reinforcement learning (DRL) proves challenging when the environmental reward signals are sparse. These limited rewards indicate whether the task has been partially or entirely accomplished, and the agent must conduct various exploration actions before the agent receives meaningful feedback. Therefore, most existing DRL exploration algorithms cannot acquire practical behavior policies within a reasonable term. This study introduces an RL method that leverages offline demonstration trajectories for a faster and more efficient online RL in environments with sparse rewards. The crucial insight we provide is to treat offline demonstration trajectories as guidance, rather than merely imitation, allowing our method to identify a policy with a distribution of state-action visitation that is marginally in line with offline demonstrations. A new trajectory distance based on maximum mean discrepancy (MMD) is presented and cast as a distance-constrained optimization problem. As a result, we demonstrate that the optimization problem can be streamlined by using a policy-gradient algorithm, incorporating rewards based on insights acquired from offline demonstrations. In this study, the proposed algorithm is evaluated across a navigation task with a discrete action space and two continuous locomotion control tasks. According to our experimental findings, our suggested algorithm offers significant advantages over baseline methods in terms of exploring diverse policy spaces and acquiring optimal policies.
\end{abstract}


\section{Introduction}
The iterative learning process of reinforcement learning (RL) has demonstrated remarkable efficacy in addressing intricate decision-making challenges~\cite{mnih2015human, mnih2016asynchronous, silver2016mastering}. Nevertheless, when confronted with sparse or delayed environmental reward signals, these RL methodologies may encounter inefficiencies in sample complexity and suboptimal performance. This primarily stems from the inherent challenge of navigating environments effectively and the limited availability of immediate reward cues~\cite{yang2021exploration, wang2023adaptive}. Additionally, real-world tasks often feature ambiguously defined objectives, making it challenging to devise precise reward functions for these endeavors~\cite{jeon2020reward}. Consequently, this practical dilemma imposes demanding criteria on DRL techniques to be viable within sparse reward scenarios. By addressing this limitation, DRL is made more applicable to real-world scenarios as well as its potential impact is significantly broadened.

Reward engineering is an alternative and direct way to provide dense and significant reward signals, which mainly refers to manually designing reward function~\cite{turner2020conservative, sumers2021learning}. Unfortunately, building a specified reward function may require extensive instrumentation, such as thermal cameras to detect liquids pouring~\cite{schenck2017visual}, and accelerometers to detect door opening~\cite{yahya2017collective}. Moreover, it may still be hard to construct a handcrafted and suitable reward function because of undue reward exploitation. Specifically, RL agents often discover unexpected and unintended ways to fulfill high returns and cater to biased objective functions~\cite{turner2020conservative}.

Another promising approach to conquering the sparse-reward difficulty is learning from demonstrations (LfD). LfD methods leverage offline expert demonstrations to enhance the performance of DRL algorithms~\cite{oh2018self,hester2018deep,wang2023learning}. Some LfD methods only regard demonstrations as data augmentations without fully utilizing them in the policy optimization process~\cite{libardi2021guided}. Many other studies accelerate policy optimization by pre-training the policy with demonstrations in a supervised manner~\cite{silver2016mastering}. Furthermore, recent LfD research proposes encouraging the agent to mimic the expert action distribution, which takes inspiration from imitation learning (IL). However, these methods often require offline demonstrations to be perfect and sufficient, burdening human experts more~\cite{zhu2022self}. 

This study develops a novel and practicable approach to encourage efficient exploration and achieve reliable credit assignments in tasks with sparse rewards. We call our method \textbf{T}rajectory \textbf{O}riented \textbf{P}olicy \textbf{O}ptimization (TOPO). Our crucial insight is that by regarding offline demonstration trajectories as guidance, rather than mimicking them, the agent is incentivized to acquire a behavior policy whose state marginal visitation distribution matches that of offline expert demonstrations. In this manner, TOPO avoids introducing complicated models or prior knowledge to obtain a reliable and feasible reward function. Specifically, this study introduces a novel distance measure between trajectories' state-action visitation distributions based on maximum mean discrepancy (MMD). We reformulate a novel trajectory-guided policy optimization problem. Subsequently, this study illustrates that a policy-gradient algorithm can be obtained from this optimization problem by incorporating intrinsic rewards derived from the distance between trajectories. Evaluation of the proposed algorithm is conducted across extensive discrete and continuous control tasks characterized by sparse and misleading rewards. Experiments conclusively show that TOPO outperforms other baseline approaches in diverse environments where rewards are sparse.

\section{Preliminaries}
\label{sec:backgroung}
\subsection{Reinforcement Learning}
In this study, we explore a discrete-time Markov decision process (MDP) represented by a tuple $(\mathcal{S}, \mathcal{A}, P, r, \rho_0, \gamma)$. Here, $\mathcal{S}$ denotes a state space, which may be discrete or continuous, while $\mathcal{A}$ represents an action space, which could also be discrete or continuous. The probability distribution of transition $P: \mathcal{S}\times\mathcal{A}\rightarrow\Pi(\mathcal{S})$ captures the transition probability distribution function between states, with $\Pi(\mathcal{S})$ denoting the space of probability distributions over states. Additionally, the reward function $r: \mathcal{S}\times\mathcal{A}\rightarrow[R_{min}, R_{max}]$ defines the rewards associated with state-action pairs, where $R_{min}$ and $R_{max}$ denote the minimum and maximum possible rewards, respectively. The initial state distribution is denoted by $\rho_0$, and $\gamma$ lies within the interval $[0,1]$ and represents the discount factor. The policy $\pi_\theta$ maps an arbitrary state in $\mathcal{S}$ to an action probability distribution function over $\mathcal{A}$, i.e. $\pi_{\theta}: \mathcal{S}\rightarrow\mathcal{P}(\mathcal{A})$, where $\mathcal{P}(\mathcal{A})$ is a collection of probability distribution function over the action space. Typically, the objective function in reinforcement learning (RL) seeks to maximize the expected discounted return, which can be expressed as follows:
\begin{equation}
    \label{eq:rl_objective}
    J(\pi_{\theta}) = \mathbb{E}_{s_0,a_0,\dots}\left[\sum_{t=0}^{\infty}\gamma^{t}r(s_{t},a_t)\right].
\end{equation}

The preceding equation entails that $s_0$ is selected randomly from the distribution $\rho_0(s_0)$, $a_t$ is chosen from $\pi_{\theta}(a_t \vert s_t)$, and $s_{t+1}$ is randomly selected based on $P(s_{t+1}\vert s_t,a_t)$. The standard definition of the state-action value function $Q$ is given as follows:
\begin{equation}
  \label{eq:Q_func}
  Q(s_t, a_t)=\mathbb{E}_{s_{t+1}, a_{t+1},\dots}\left[\sum_{k}^{\infty}\gamma^{k} r(s_{t+k},a_{t+k})\right],
\end{equation}
and the value function of the state is expressed as:
\begin{equation}
  \label{eq:V_func}
  V(s_t)=\mathbb{E}_{a_t, s_{t+1}, a_{t+1},\dots}\left[\sum_{k}^{\infty}\gamma^{k} r(s_{t+k},a_{t+k})\right],
\end{equation}
Within this context, $s_0$ is selected in accordance with $\rho_0(s_0)$, $a_t$ is chosen from $\pi_{\theta}(a_t \vert s_t)$, and $s_{t+1}$ is determined based on $P(s_{t+1}\vert s_t,a_t)$. Following this, we obtain the definition of the advantage function:
\begin{equation}
  \label{eq:adv_func}
  A(s_t, a_t) = Q(s_t, a_t) - V(s_t).
\end{equation}


\subsection{Maximum Mean Discrepancy}
The Maximum Mean Discrepancy (MMD) emerges as a pivotal metric for probability assessment, discerning the disparity between any pair of probability distributions~\cite{gretton2006kernel, Gretton2012OptimalKC}. Suppose $p$ and $q$ represent two probability distribution functions defined in a space $\mathbb{X}$. Let $x$ symbolize the element independently drawn from $p$ and $y$ represent the element from the distribution $q$, respectively. Subsequently, the precise definition of the MMD measure between $p$ and $q$ is expressed as follows:

When $k(\cdot.\cdot)$ is the kernel function of a Reproducing Kernel Hilbert Space (RKHS) $\mathcal{H}$. In practice, the kernel function can chosen as a Gaussian or Laplace kernel function. Following this, the definition of the MMD measure between distributions $p$ and $q$ is then specified, as referenced by~\cite{Gretton2012OptimalKC, wang2023policy}.
\begin{equation}
  \label{equ: MMD_rkhs}
  \begin{aligned}
    \mathrm{MMD}^2(p, q, \mathcal{H}) = \mathbb{E}[k(x, x^\prime)] - 2\mathbb{E}[k(x, y)] + \mathbb{E}[k(y, y^\prime)],
  \end{aligned}
\end{equation}
where $x, x^\prime\ \mathrm{i.i.d.}\sim p$ and $y, y^\prime\ \mathrm{i.i.d.} \sim q$.

\section{Proposed Approach}
\label{sec:approach}
In this section, we introduce a method for trajectory-guided exploration in DRL, aiming to tackle the issue of challenging exploration. To tackle this problem, we reformulate a novel constrained policy optimization problem and derive a reliable intrinsic reward function based on the proposed MMD distance. Our method can promote the adjustment of policy parameters to generate a policy whose state-action visitation aligns with the expert's demonstrations.

\subsection{Trajectory-Guided Exploration Strategy}
\label{sec:optimization_problem}
We suggest reformulating a unique constrained optimization problem to promote effective exploration in environments with sparse rewards. Assume a set $\mathcal{M}$ comprises offline demonstration trajectories leading to imperfect sparse rewards, indicating that the optimal policy may not generate these demonstrations. Our aim is to direct the agent towards exploration within the specified region of the state space while following the demonstration trajectories, thus minimizing unnecessary exploration. One strategy to achieve this goal involves reducing the disparity between the current trajectories and the demonstration trajectories. To quantify this disparity, we utilize the squared $\mathrm{MMD}$ distance measure, defined in~\eqref{equ: MMD_rkhs}, which measures the difference between various trajectories.

During each training iteration, multiple trajectories are generated based on $\pi_\theta$ and maintained in the on-policy buffer $\mathcal{B}$. We treat each trajectory $\tau$ within $\mathcal{B}$, as well as the offline demonstration dataset $\mathcal{M}$, as a deterministic behavior policy. Subsequently, the state-action visitation distribution $\rho_\tau$ is computed for each trajectory. We then calculate the squared $\mathrm{MMD}$ distance between the different state marginal distributions of the expert and current behavioral policies. To be precise, following~\eqref{equ: MMD_rkhs}, the squared MMD distance between the distributions of state-action within trajectories is expressed as:
\begin{equation}
  \label{equ:MMD_traj}
  \begin{aligned}
  \mathrm{MMD}^2(\tau, \upsilon, \mathcal{H}) &= \underset{x, x^\prime \sim \rho_\tau}{\mathbb{E}} \left[k\left(x, x^\prime\right)\right] \\ 
  &- 2\underset{\begin{subarray}{c}x \sim \rho_\tau \\ y \sim \rho_\upsilon\end{subarray}}{\mathbb{E}} \left[k(x, y)\right] \\
  &+ \underset{y, y^\prime \sim \rho_\upsilon}{\mathbb{E}} \left[k(y, y^\prime)\right].
  \end{aligned}
\end{equation}

\renewcommand{\algorithmicrequire}{\textbf{Input:}}
\renewcommand{\algorithmicensure}{\textbf{Output:}}
\begin{algorithm}[htb]
\caption{TOPO based on PPO}
\label{algo:tgppo}
\begin{algorithmic}[1]
  \REQUIRE policy update frequency $K$, expert demonstration dataset $\mathcal{M}$, episode number $N$, step number of each episode $T$
  \STATE Initialize network parameters $\theta$
  \STATE \textsc{//expert demonstration buffer}
  \STATE Initialize the demonstration buffer $\mathcal{M} \leftarrow \{\tau_{demo}\}$ 
  \FOR{$\text{each episode $n$}\in\{0,\dots,N\}$}
  \STATE Initialize the agent observation $\boldsymbol{s}_0$
  \STATE $\tau_n=\{\}$\quad\quad\quad\quad \textsc{// The set maintains the current experience}
  \STATE{$R_n = 0$}\quad\quad\quad\quad \textsc{// The variable records accumulate rewards}
  \FOR{each time step $t$ in $\{1,\dots,T\}$}
  \STATE sample a possible action $\boldsymbol{a}_t$ according to $\pi_\theta(\cdot|\boldsymbol{s}_t)$
  \STATE perform selected action $\boldsymbol{a}_t$ and obtain $\boldsymbol{r}^e_t$ and $\boldsymbol{s}^\prime_{t+1}$
  \STATE $\tau_n\leftarrow\tau_n\cup\{(\boldsymbol{s}_t, \boldsymbol{a}_t)\}$; $R_n\leftarrow R + \boldsymbol{r}^e_t$
  \ENDFOR
  \STATE Calculate the MMD distance $D_{\rm MMD}(x, \mathcal{M})$ for each state-action pair $x=(\boldsymbol{s},\boldsymbol{a})$ of this episode
  \IF{$\text{episode}\, \%\, K == 0$}
  \STATE Estimate the MMD gradient $\nabla_{\theta}D_{\rm MMD}$ using $\mathcal{M}$ and $\mathcal{B}$
  \STATE Estimate the policy gradient $\nabla_{\theta}J$ based on $\mathcal{B}$
  \STATE Calculate the final gradient $\nabla_{\theta}L = \nabla_{\theta} J - \sigma \nabla_\theta D_{\rm MMD}$
  \STATE \textsc{// Update the policy parameter}
  \STATE $\theta \leftarrow \theta + \alpha\nabla_{\theta}L$
  \ENDIF
  \ENDFOR
\end{algorithmic}
\end{algorithm}
Within this framework, $\tau$ is a member of $\mathcal{B}$, $\upsilon$ is an element of $\mathcal{M}$, and $x, x^\prime, y, y^\prime$ denote state-action pairs. $\rho_\tau$ and $\rho_\upsilon$ are used to represent the state-action marginal distributions of the trajectories $\tau$ and $\upsilon$, respectively. Furthermore, $k(\cdot, \cdot)$ in Eq.~\eqref{equ:MMD_traj} is defined as:
\begin{equation}
  \label{eq:k}
  k(x, y) = K\left(g(x), g(y)\right).
\end{equation}
In this scenario, $K(\cdot, \cdot)$ represents the kernel function of a reproducing kernel Hilbert space $\mathcal{H}$. Typically, we opt for the Gaussian kernel in practical applications.
\begin{figure*}[ht]
  \centering
  \subfloat[]{
  \includegraphics[width=5.2cm]{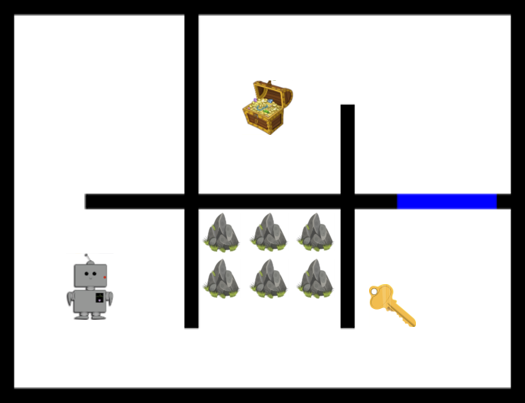}
  \label{fig:kdt_maze}
  }
  \subfloat[]{
    \includegraphics[width=4cm]{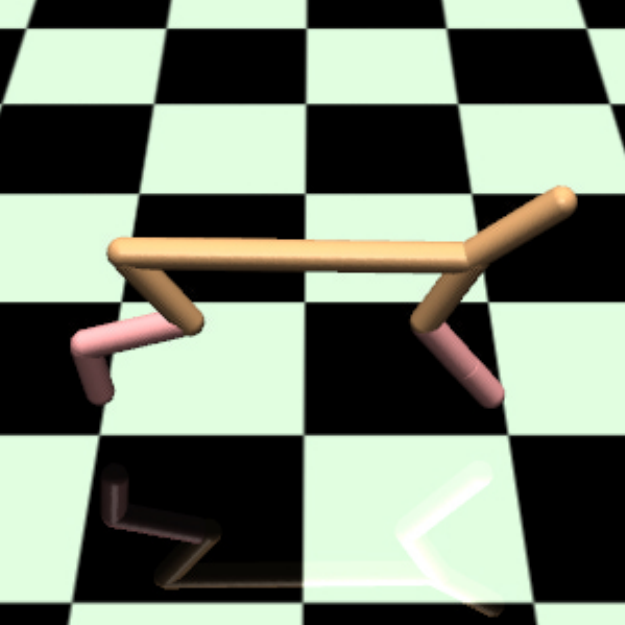}
    \label{fig:halfcheetah_envs}
  }
  \subfloat[]{
    \includegraphics[width=4cm]{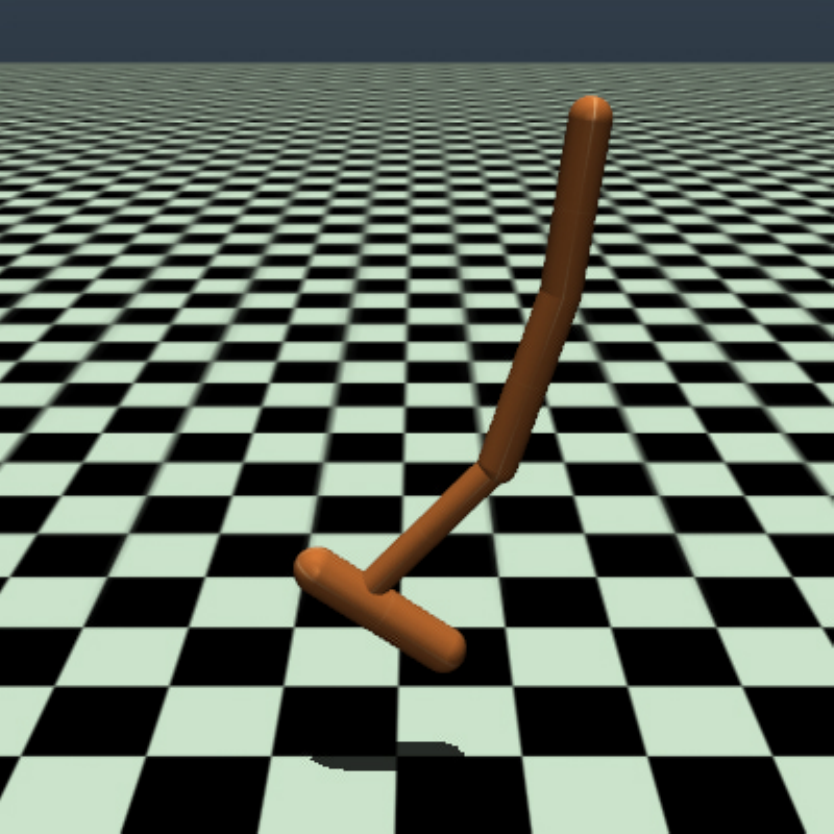}
    \label{fig:hopper_envs}
  }
  \caption{(a) Key-Door-Treasure domain; (b) SparseHalfCheetah; (c) SparseHopper.}
  \label{fig:conti_envs}
\end{figure*}

The function $g$ described in~\eqref{eq:k} enables versatile adjustment of focus within the squared ${\rm MMD}$ distance concerning different facets, including state visits, action selections, or both. Hence, the MMD measure of each state-action pair is computed with its information subset. Specifically, this information subset can be identified as the coordinate $c$ representing the center of mass (CoM), indicating that function $g$ assigns a state-action pair $(s, a)$ to the coordinate $c$. 
\begin{equation}
  \label{eq:D}
  D(x, \mathcal{M}) = \underset{\tau\in\mathcal{B}_x}{\mathbb{E}}\left[\mathrm{MMD}^2(\tau, \mathcal{M}, \mathcal{H})\right].
\end{equation}
%
In this context, $\mathcal{B}_x = \{\tau \vert x\in\tau, \tau\in\mathcal{B}\}$, and the squared $\mathrm{MMD}^2(\tau, \mathcal{M}, \mathcal{H})$ is delineated as:
\begin{equation}
  \mathrm{MMD}^2(\tau, \mathcal{M}, \mathcal{H}) = \min_{\upsilon\in\mathcal{M}}\mathrm{MMD}^2(\tau, \upsilon, \mathcal{H}).
\end{equation}

To emphasize the reliance of the distance calculation in Eq.~\eqref{eq:D} on the maximum mean discrepancy, we introduce $\mathrm{MMD}$ as the subscript of the sign $D$. The stochastic optimization problem, incorporating constraints associated with the $\mathrm{MMD}$ distance, is delineated as follows:
\begin{equation}
  \label{eq: constraint}
  \begin{aligned}
    &\max_{\theta} \, J(\theta), \\
    &s.t.\ D_{\rm MMD}(x, \mathcal{M}) \le \delta, \quad \forall x \in \mathcal{B}.
  \end{aligned}
\end{equation}
Here, $J$ denotes a standard objective in reinforcement learning, while $\delta$ stands for a constant distance boundary value.

\begin{remark}
  During policy optimization, the updating of replay memory $\mathcal{M}$ occurs. If a current trajectory surpasses trajectories within $\mathcal{M}$—for instance, demonstrating a higher return—it becomes viable to replace the trajectory with the lowest return in $\mathcal{M}$ with the superior one. Furthermore, initial offline trajectories may originate from human experts. The core premise of our investigation centers on aligning its policy closely with offline demonstrations by considering them as soft guidance. Unlike RLfD techniques, our method doesn't require flawless and abundant demonstrations, thereby presenting a more realistic scenario.
\end{remark}

\subsection{Practical Algorithms}
To address the optimization problem posed in~\eqref{eq: constraint}, we convert it into an unconstrained objective function, yielding the subsequent expression:
\begin{equation}
\label{equ: P_I}
  \begin{aligned}
    L(\theta, \sigma) = J(\theta)
    - \sigma\underset{x\sim\rho_\pi}{\mathbb{E}}\left[\max\left\{D_{\rm MMD}(x, \rho_\mu) - \delta, 0\right\}\right].
  \end{aligned}
\end{equation}
Here, the Lagrange multiplier $\sigma$ is a positive value.

Afterward, we calculate the policy gradient for the aforementioned unconstrained optimization task. The first part of the unconstrained problem signifies the conventional RL objective function, making the computation of its gradient straightforward~\cite{schulman2017proximal}. Next, we derive the gradient regarding policy parameters for the ${\rm MMD}$ component, enabling the effective optimization of the policy. The result is expressed in the following lemma.

\begin{restatable}{lemma}{mmdgradient}
  \label{lemma:1}
  Let $\rho_{\pi}(s, a)$ represent the state-action marginal visitation distribution function produced by the current behavior policy $\pi_\theta$. Let $D(x,\mathcal{M})$ denote the MMD distance measure of the current state-action pair $x$ to the offline demonstration buffer $\mathcal{M}$. Then, the gradient regarding $\theta$ of the ${\rm MMD}$ component is obtained as follows:
  \begin{equation}
    \label{equ:nabla_E_D_mmd}
    \begin{aligned}
      &\nabla_{\theta} \mathbb{E}\left[\max\left\{D_{\rm MMD}(x, \rho_\mu) - \delta, 0\right\}\right]\\
      =&\underset{\rho_{\pi}(s,a)}{\mathbb{E}}\left[\nabla_{\theta}\log\pi_{\theta}(a \vert s)Q_i(s,a)\right],
    \end{aligned}
  \end{equation}
  where
  \begin{equation}
    \label{equ:Q}
    Q^{(i)}(s_t, a_t) = \underset{\rho_{\pi}(s, a)}{\mathbb{E}}\left[\sum_{l=0}^{T-t}\gamma^{l}r^{(i)}(s, a)\right],
  \end{equation}
  and
  \begin{equation}
      r^{(i)}(s, a) = \max\left\{D_{\rm MMD}(x, \mathcal{M})-\delta, 0\right\}.
  \end{equation}
\end{restatable}
\begin{proof}
Let the symbol $x$ represent a current state-action pair, i.e. $x = (s, a)$. We define a new intrinsic reward function $r^{(i)}(s, a)$ derived from the MMD distance as follows:
\begin{equation}
r^{(i)}(s, a) = \max\left\{D_{\rm MMD}(x, \mathcal{M}) - \delta, 0\right\}.
\end{equation}
Moreover, consider $Q^{(i)}(\cdot, \cdot)$ as the $Q$-function calculated based on the distance-induced reward function $r^{(i)}(s, a)$:
\begin{equation}
Q^{(i)}(s_t, a_t) = \mathbb{E}\left[\sum_{l=0}^{T-t}\gamma^{l}r^{(i)}(s_{t+l}, a_{t+l})\right].
\end{equation}
\begin{figure*}[ht]
  \centering
  \subfloat[]{
    \includegraphics[width=7cm]{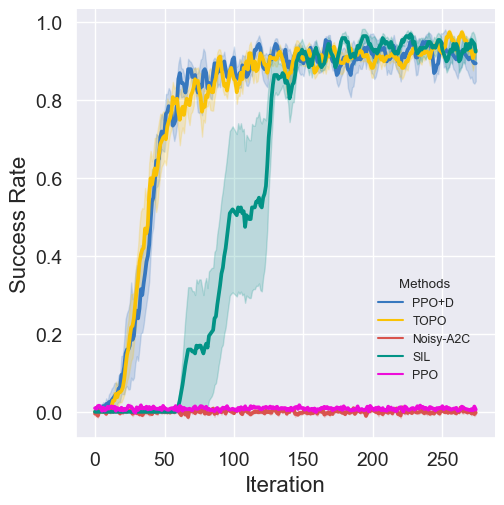}
    \label{fig:kdt_rate}
  }
  \subfloat[]{
    \includegraphics[width=7cm]{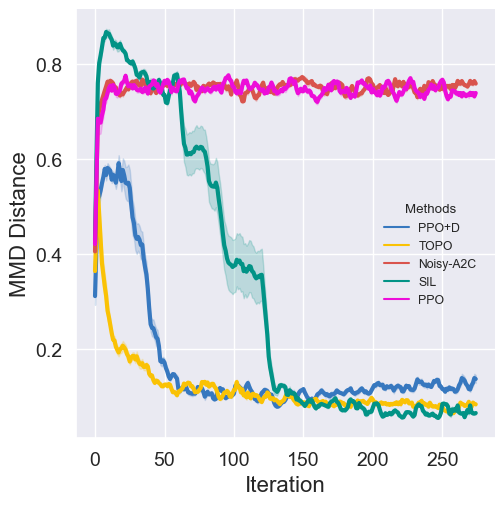}
    \label{fig:kdt_mmd}
  }
  \caption{Evaluation of TOPO in the discrete Key-Door-Treasure task: (a) The learning curves of success rate; (b) The trend of the MMD distance between the current policy and demonstrations;}
  \label{fig:kdt_results}
\end{figure*}
We can utilize the policy gradient theorem~\cite{sutton1999policy} to calculate the gradient of~\eqref{equ: P_I} easily.
\begin{equation}
\begin{aligned}
&\nabla_{\theta} \mathbb{E}\left[\max\left\{D_{\rm MMD}(x, \rho_\mu) - \delta, 0\right\}\right]\\
=&\underset{\rho_{\pi}(s,a)}{\mathbb{E}}\left[\nabla_{\theta}\log\pi_{\theta}(a \vert s)Q_i(s,a)\right].
\end{aligned}
\end{equation}
For simplicity, we omit the specific derivation process.
\end{proof}

\begin{figure*}[ht]
\centering
\subfloat[]{
  \includegraphics[width=7cm]{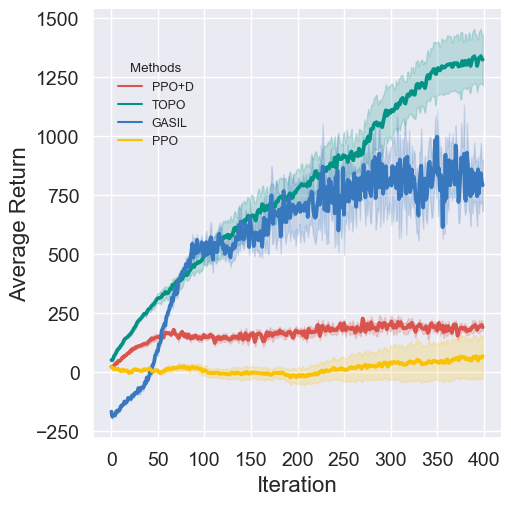}
  \label{fig:cheetah_return}
}
\subfloat[]{
  \includegraphics[width=7cm]{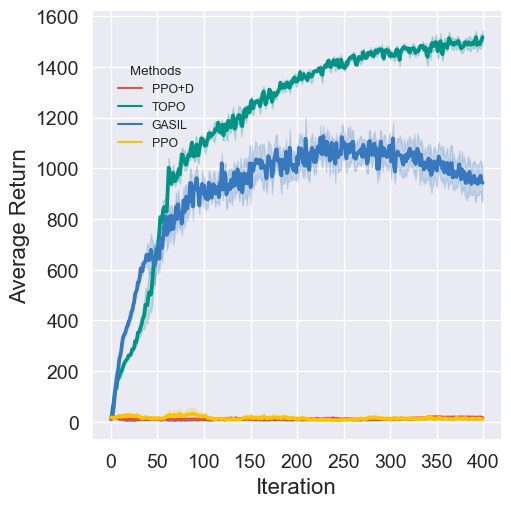}
  \label{fig:hopper_return}
}
\caption{Evaluation of TOPO on the SparseHalfCheetah and SparseHopper task.}
\label{fig:locomotion_results}
\end{figure*}

\section{Experiments}
This section presents a thorough assessment of the TOPO algorithm. We evaluate TOPO's performance across continuous and discrete control tasks, as outlined in Section~\ref{sec: setup}. Our experimental results demonstrate that TOPO outperforms alternative baseline methods regarding average return and learning rate.

\subsection{Experimental Settings}
\label{sec: setup}
The initial assessment of TOPO's performance was conducted in the Key-Door-Treasure task, as illustrated in Fig.~\ref{fig:kdt_maze}. This task's state-action space is discrete, and the grid-world maze dimensions are $26\times 36$. In each episode, the agent commences from a fixed starting point in the bottom-left room. The episode's duration is fixed, terminating immediately upon the agent discovering the treasure. A positive reward 200 is exclusively granted when the treasure is reached, and no rewards are provided in other instances. In each episode, the RL agent first observes its environmental information and then generates a possible action based on its current policy. The set of all possible actions is \textit{go south, north, east,} or \textit{west}. The agent should complete the following three tasks sequentially to obtain the treasure: collecting the door key (K), opening the door (D) using the key, and then entering the upper-right room and obtaining the treasure (T). TOPO's performance was compared against four baseline methods: PPO~\cite{schulman2017proximal}, SIL~\cite{oh2018self}, PPO+D~\cite{libardi2021guided}, and Noisy-A2C~\cite{fortunato2019noisy}.
%

As depicted in Figs.~\ref{fig:halfcheetah_envs} and~\ref{fig:hopper_envs}, to further demonstrate the TOPO's superior advantage in more challenging environments, two established MuJoCo locomotion control agents, HalfCheetah and Hopper, were adapted. In this manner, we yielded two novel agents: SparseHalfCheetah and SparseHopper. These agents exclusively receive rewards based on forward velocity when the centers of the robots' mass have advanced toward a specific direction beyond a predetermined threshold distance; otherwise, no positive rewards are accrued. Specifically, the predetermined threshold distance for SparseCheetah is set at 10 units, while 1 unit for SparseHopper. During training, the agent obtains the environmental observations and performs a possible action outputted by the current policy. Additionally, the movement of the agent can result in energy punishment, which encourages the agent to minimize joint control torque to decrease the energy loss. TOPO's performance was compared against three state-of-the-art baseline methods: PPO~\cite{schulman2017proximal}, PPO+D~\cite{libardi2021guided}, and GASIL~\cite{guo2018generative}.
%

\subsection{Results in the Key-Door-Treasure domain}
Depicted in Fig.~\ref{fig:kdt_maze}, in the context of the key-door-treasure scenario, obtaining the ultimate treasure requires the agent to sequentially accomplish three specific tasks: obtaining the door key, unlocking and opening the door, and ultimately obtaining the treasure. It is important to note that, for fair evaluation, the same expert demonstrations were provided for SIL at the beginning of training. The results are presented in Fig.~\ref{fig:kdt_rate}. The PPO agent, often entangled in a sub-optimal policy solely focused on the key acquisition, fails to achieve optimal rewards upon reaching the treasure. Noisy-A2C exhibits analogous training outcomes to PPO. In contrast, the robust baselines, SIL and PPO+D, demonstrate accelerated learning facilitated by the exploration bonus derived from the demonstration data, leading them to discover the treasure. Notably, TOPO attains competitive training performance comparable to PPO+D, surpassing SIL in this task. This outcome signifies TOPO's adeptness in the efficient exploration of environments by leveraging demonstration trajectories to redefine a constrained optimization problem. Fig.~\ref{fig:kdt_mmd} illustrates TOPO's effective reduction of the MMD distance between the current policy and expert trajectories.

\subsection{Comparisons on locomotion control tasks}
As demonstrated in Figs.~\ref{fig:cheetah_return} and~\ref{fig:hopper_return}, TOPO outperforms other baseline methods in the SparseHopper and SparseCheetah tasks. Specifically, TOPO demonstrates accelerated learning during policy optimization, attaining a higher final return after training. GASIL focuses solely on disparities between the agent's behavior policy and expert trajectories, constructing a reward function based on this discrepancy. However, this approach places a premium on sample quality, limiting its competitiveness against TOPO. PPO+D relies exclusively on the advantage information from demonstrations for policy gradient computation, resulting in inefficient policy optimization with offline demonstrations. This discrepancy in approach contributes to the performance gap between PPO+D and TOPO. Noteworthy is the stagnation in the changing trend of PPO+D's average return on the SparseHopper task throughout training. This fact underscores TOPO's efficacy in exploration by incorporating distribution information from sparse-reward demonstrations and defining a dense reward function based on the MMD distance.

\section{Conclusions}
Presented in this study is TOPO, a method devised to expedite and enhance online Reinforcement Learning (RL) through the utilization of offline demonstrations in tasks characterized by sparse rewards. TOPO eliminates the necessity for complex reward function models, thereby reducing the dependency on a large volume of demonstrations. The central idea revolves around training a policy to align its state-action visitation distribution with offline demonstrations, treating these trajectories as instructive guides. Specifically, a novel metric for trajectory distance based on Maximum Mean Discrepancy (MMD) is introduced and the policy optimization problem is reformulated as a distance-constrained optimization task. Consequently, this distance-constrained optimization problem can be transformed into a policy-gradient algorithm, which incorporates intrinsic distance rewards obtained from expert trajectories. The proposed algorithm undergoes extensive evaluation across several benchmark control tasks. The experimental results highlight TOPO's superiority over baseline methods regarding the ability to avoid local optima and achieve diverse exploration.

\bibliographystyle{IEEEtran_bst}
\bibliography{reference}

\end{document}